\documentclass[a4paper,10pt]{article}

\usepackage{amsmath,amssymb,amsthm,enumerate}
\usepackage{graphicx}
\def\titre{Rate of Convergence and Error Bounds for LSTD($\lambda$)}
\newtheorem{thm}{Theorem}
\newtheorem{cor}{Corollary}
\newtheorem{lemma}{Lemma}

\newtheorem{remark}{Remark}

\newtheorem{assm}{Assumption}
\def\M{{\cal M}}
\def\X{{\cal X}}

\def\B{{\cal B}}
\def\S{{\cal S}}
\def\P{\mathbb{P}}
\def\E{\mathbb{E}}
\def\R{\mathbb{R}}

\usepackage[utf8]{inputenc}

\usepackage{fullpage}

\usepackage{natbib}

\usepackage[usenames]{color} 

\title{\titre}

\author{Manel Tagorti and Bruno Scherrer \\
INRIA Nancy Grand Est, Team MAIA\\
manel.tagorti@inria.fr, bruno.scherrer@inria.fr}

\begin{document}

\maketitle

\begin{abstract}%
We consider LSTD($\lambda$), the least-squares temporal-difference algorithm with eligibility traces algorithm proposed by Boyan (2002). It computes a linear approximation of the value function of a fixed policy in a large Markov Decision Process. Under a $\beta$-mixing assumption, we derive, for any value of $\lambda \in (0,1)$,  a high-probability estimate of the rate of convergence of this algorithm to its limit. We deduce a high-probability bound on the error of this algorithm, that extends (and slightly improves) that derived by Lazaric et al. (2010) in the specific case where $\lambda=0$. In particular, our analysis sheds some light on the choice of $\lambda$ with respect to the quality of the chosen linear space and the number of samples, that complies with simulations.

\end{abstract}

\section{Introduction}
In a large Markov Decision Process context, we consider LSTD($\lambda$), the least-squares
temporal-difference algorithm with eligibility traces
proposed by \citet{Boyan:2002}. It is a popular algorithm for estimating a
projection onto a linear space of the value function of a fixed
policy. Such a value estimation procedure can for instance be useful
in a policy iteration context to eventually estimate an approximately optimal controller \citep{ndp,reviewszepes}.

The asymptotic almost sure convergence of LSTD($\lambda$) was proved by
\citet{Nedic02leastsquares}.  Under a $\beta$-mixing assumption, and given a finite number of samples $n$, \citet{lazaric:2012}
derived a high-probability error bound with a $\tilde O(\frac{1}{\sqrt{n}})$ rate\footnote{Throughout the paper, we shall write $f(n) = \tilde O(g(n))$ as a  shorthand for $f(n) = O(g(n) \log^k g(n))$ for some $k \ge 0$.}
in the restricted situation where
$\lambda=0$. To our knowledge, however, similar finite-sample
error bounds are not known in the
literature for $\lambda>0$. The main goal of this paper is to fill this gap.  This is
all the more important that it is known that the parameter $\lambda$
allows to control the quality of the asymptotic solution of the value: by moving $\lambda$ from $0$ to $1$, one can continuously move from an oblique projection of the value \citep{scherrer:2010} to its orthogonal projection
and consequently improve the corresponding guarantee \citep{vanroy:1997} (restated in Theorem \ref{thm:errapp}, Section~\ref{section:main}).

The paper is organized as follows. Section~\ref{section:2} starts by describing the LSTD($\lambda$) algorithm and the necessary background. Section~\ref{section:main} then contains our main result (Theorem~\ref{thm:main}): for all $\lambda \in (0,1)$, we will show that LSTD($\lambda$) converges to its limit at a rate $\tilde O(\frac{1}{\sqrt{n}})$. We shall then deduce a global error (Corollary \ref{cor:globalerror}) that sheds some light on the role of the parameter $\lambda$ and discuss some of its interesting practical consequences. Section~\ref{section:proof} will go on by providing a detailed proof of our claims. Finally, Section \ref{section:4} concludes and describes potential future work.

\section{LSTD($\lambda$) and Related background}\label{section:2}

We consider a Markov chain $\M$ taking its values on a finite or countable state space\footnote{We restrict our focus to finite/countable mainly because it eases the presentation of our analysis. Though this requires some extra work, we believe the analysis we make here can be extended to more general state spaces.} $\X$, with transition kernel $P$. 
We assume $\M$ ergodic\footnote{In our countable state space situation, ergodicity holds if and only if the chain is aperiodic and irreducible, that is formally if and only if: $\forall (x,y)\in \X^2,~\exists n_0,~\forall n\geq n_0,~P^n(x,y)>0.$}; 
consequently, it  admits a unique stationary distribution $\mu$. 
For any $K\in \R$, we denote $\B(\X,K)$ the set of measurable functions defined on $\X$ and bounded by $K$. We consider a reward function $r \in \B(\X,R_{\text{max}})$ for some $R_{\text{max}} \in \R$, that provides the quality of being in some state.
The value function $v$ related to the Markov chain $\M$ is defined, for any state $i$,  as the average discounted sum of rewards along infinitely long trajectories starting from $i$:
\[\forall i\in \X,~ v(i)=\E\left[\sum^\infty_{j=0}\gamma^jr(X_j)\left|\vphantom{\sum^\infty_{j=0}\gamma^jr(X_j)}X_0=i\right.\right],\]
where $\gamma \in (0,1)$ is a discount factor.
It is well-known that the value function  $v$ is the unique fixed point of the linear Bellman operator~$T$:
\[\forall i\in \X,~ Tv(i)=r(i)+\gamma \E \left[v(X_1)|X_0=i \right].\]
It can easily be seen that $v \in \B(\X,V_{\text{max}})$ with $V_{\text{max}}=\frac{R_{\text{max}}}{1-\gamma}$.

When the size $|\X|$ of the state space is very large, one may consider approximating $v$ by using \emph{a linear architecture}.
Given some $d \ll |\X|$, we consider a feature matrix $\Phi$ of dimension $|\X| \times d$. For any $x \in \X$,  $\phi(x)=(\phi_1(x),...,\phi_d(x))^T$ is the \emph{feature vector} in state $x$. For any $j\in\{1,...,d\}$, we assume that the \emph{feature function} $\phi_j:\X \mapsto \R$ belongs to $\B(\X,L)$ for some finite $L$. Throughout the paper, and without loss of generality\footnote{This assumption is not fundamental: in theory, we can remove any set of features that makes the family linearly dependent; in practice, the algorithm we are going to describe can use the pseudo-inverse instead of the inverse.} we will make the following assumption.
\begin{assm}
\label{assm:indfeat}
The feature functions $(\phi_j)_{j \in \{1,...,d\}}$ are linearly independent. 
\end{assm}
Let $\S$ be the subspace generated by the vectors $(\phi_j)_{1 \le j
  \le d}$.  We consider the orthogonal projection $\Pi$ onto $\S$ with
respect to the $\mu$-weighed quadratic norm
$$\|f\|_\mu=\sqrt{\sum_{x\in \X} |f(x)|^2\mu(x)}.$$  It is well known
that this projection has the following closed form
\begin{align}
\label{eq:defPi}
\Pi=\Phi(\Phi^TD_\mu\Phi)^{-1}\Phi^TD_\mu,
\end{align} where $D_\mu$ is the
diagonal matrix with elements of $\mu$ on the diagonal.

The goal of LSTD($\lambda$) is to estimate a solution of the equation $v=\Pi T^{\lambda}v$, where the operator $T^\lambda$ is defined as a weighted arithmetic mean of the applications of the powers $T^i$ of the Bellman operator $T$ for all $i>1$:
\begin{align}
\label{eq:defTlambda}
\forall \lambda\in (0,1),~\forall v,~ T^{\lambda}v=(1-\lambda)\sum^{\infty}_{i=0}{\lambda^iT^{i+1}v}.
\end{align}
Note in particular that when $\lambda=0$, one has $T^\lambda=T$.
By using the facts that $T^i$ is affine and $\|P\|_\mu=1$ \citep{vanroy:1997,Nedic02leastsquares}, it can be seen that the operator $T^{\lambda}$ is a contraction mapping of modulus $\frac{(1-\lambda)\gamma}{1-\lambda\gamma} \le \gamma$; indeed,  
for any vectors $u,v$: 
\begin{align*}
\|T^{\lambda}u-T^{\lambda}v\|_\mu&\leq (1-\lambda)\|\sum^{\infty}_{i=0}{\lambda^i(T^{i+1}u- T^{i+1} v)}\|_\mu\\
& = (1-\lambda)\|\sum^{\infty}_{i=0}{\lambda^i(\gamma^{i+1}P^{i+1}u- \gamma^{i+1}P^{i+1} v)}\|_\mu \\
& \leq (1-\lambda)\sum^{\infty}_{i=0}{\lambda^i\gamma^{i+1} \|u - v\|_\mu}\\
&= \frac{(1-\lambda)\gamma}{1-\lambda\gamma}\|u-v\|_\mu.
\end{align*}
Since the orthogonal projector $\Pi$ is non-expansive with respect to $\mu$  \citep{vanroy:1997}, the operator $\Pi T^{\lambda}$ is contracting and thus the equation $v=\Pi T^{\lambda}v$ has one and only one solution, which we shall denote  $v_{LSTD(\lambda)}$ since it is what the LSTD($\lambda$) algorithm converges to \citep{Nedic02leastsquares}. As $v_{LSTD(\lambda)}$ belongs to the subspace $\S$, there exists a $\theta \in \R^d$ such that
\begin{align*}
v_{LSTD(\lambda)}&=\Phi\theta=\Pi T^{\lambda}\Phi\theta.
\end{align*}
If we replace $\Pi$ and $T^{\lambda}$ with their expressions (Equations \ref{eq:defPi} and \ref{eq:defTlambda}), it can be seen that $\theta$ is a solution of the equation $A\theta=b$ \citep{Nedic02leastsquares}, such that for any $i$,
\begin{align}
 A&=\Phi^T D_\mu(I-\gamma P)(I-\lambda\gamma P)^{-1}\Phi  =\E_{X_{-\infty} \sim \mu} \left[\sum^i_{k=-\infty}(\gamma\lambda)^{i-k}\phi(X_{k})(\phi(X_i)-\gamma\phi(X_{i+1}))^T \right] \label{eq:defA}\\
\mbox{and~~~} b& =\Phi^TD_\mu(I-\gamma\lambda P)^{-1}r=\E_{X_{-\infty} \sim \mu} \left[\sum^i_{k=-\infty}(\gamma\lambda)^{i-k}\phi(X_k)r(X_i)\right], \label{eq:defb}
\end{align}
where $u^T$ is the transpose of $u$. Since for all $x$, $\phi(x)$ is of dimension $d$, we see that $A$ is a $d \times d$ matrix and $b$ is a vector of size $d$.
Under Assumption~\ref{assm:indfeat}, it can be shown \citep{Nedic02leastsquares} that the matrix $A$ is invertible, and thus $v_{LSTD(\lambda)}=\Phi A^{-1}b$ is well defined.

The LSTD($\lambda$) algorithm that is the focus of this article is now precisely described. Given one trajectory $X_1,....,X_n$ generated by the Markov chain, 
the expectation-based expressions of $A$ and $b$ in Equations~\eqref{eq:defA}-\eqref{eq:defb} suggest to compute the following estimates:
\begin{align}
\hat A& =\frac{1}{n-1}\sum^{n-1}_{i=1} z_i(\phi(X_i)-\gamma\phi(X_{i+1}))^T \nonumber\\
\mbox{and~~~}\hat b &=\frac{1}{n-1}\sum^{n-1}_{i=1} z_ir (X_i)\nonumber\\
\label{eq:trace} \mbox{where~~~}z_i & =\sum^{i}_{k=1}{(\lambda\gamma)^{i-k}\phi(X_k)}
\end{align} 
is the so-called \emph{eligibility trace}.
The algorithm then returns $\hat v_{LSTD(\lambda)}=\Phi \hat \theta$ with\footnote{We will see in Theorem~\ref{thm:main} that $\hat A$ is invertible with high probability for a sufficiently big $n$.} $\hat \theta=\hat A^{-1}\hat b$, which is a (finite sample) approximation of $v_{LSTD(\lambda)}$. 
Using a variation of the law of large numbers, \citet{Nedic02leastsquares} showed that both $\hat A$ and $\hat b$ converge almost surely respectively to $A$ and $b$, which implies that $\hat v_{LSTD(\lambda)}$ tends to $v_{LSTD(\lambda)}$. The main goal of the remaining of the paper is to deepen this analysis: we shall estimate the rate of convergence of $\hat v_{LSTD(\lambda)}$ to $v_{LSTD(\lambda)}$, and bound the approximation error $\|\hat v_{LSTD(\lambda)}-v\|_\mu$ of the overall algorithm.

\section{Main results}\label{section:main}

This section contains our main results. Our key assumption for the analysis is that the Markov chain process that generates the states has some mixing property\footnote{A stationary ergodic Markov chain is \emph{always} $\beta$-mixing.}.
\begin{assm}
\label{assm:beta}
The process ${(X_n)}_{n\geq 1}$ is  $\beta$-mixing, in the sense that its $i^{\text{th}}$ coefficient
\begin{align*}
\beta_i=\sup_{t\geq 1}\E\left[\sup_{B\in \sigma(X^\infty_{t+i})}\left|P(B|\sigma(X^t_1))-P(B)\right|\right]
\end{align*}
tends to $0$ when $i$ tends to infinity, where $X^j_l=\{X_l,...,X_j\}$ for $j\geq l$ and $\sigma(X^j_l)$ is the sigma algebra generated by $X^j_l$ .
Furthermore, ${(X_n)}_{n\geq 1}$ mixes at an exponential decay rate with parameters $\overline \beta>0$, $b>0$, and $\kappa>0$ in the sense that $\beta_i\leq \overline{\beta}e^{-bi^{\kappa}}$.
\end{assm}  
Intuitively the $\beta_i$ coefficients measure the degree of dependence of samples separated by $i$ times step (the smaller the coefficient the more independence).
We are now ready to state the main result of the paper, that provides a rate of convergence of LSTD($\lambda$).

\begin{thm}
\label{thm:main}
Let Assumptions~\ref{assm:indfeat} and \ref{assm:beta} hold and let $X_1 \sim \mu$. For any $n \ge 1$ and $\delta \in (0,1)$, define:
\begin{align*}
I(n,\delta)&=32\Lambda(n,\delta)\max\left\{\frac{\Lambda(n,\delta)}{b},1\right\}^{\frac{1}{\kappa}}\\
\text{where}~~ \Lambda(n,\delta)&=\log\left(\frac{8n^2}{\delta}\right)+\log(\max\{4e^2,n\overline\beta\}).
\end{align*}
 Let $n_0(\delta)$ be the smallest integer such that
\begin{align}
\label{caracn0}
\forall n\geq n_0(\delta),~\frac{2dL^2}{(1-\gamma)\nu}\left[\frac{2}{\sqrt{n-1}}\sqrt{\left(\left\lceil\frac{\log(n-1)}{\log\left(\frac{1}{\lambda\gamma}\right)}\right\rceil+1\right)I(n-1,\delta)}+\right.\nonumber\\
\left.\frac{1}{(n-1)(1-\lambda\gamma)}+\frac{2}{(n-1)}\left\lceil\frac{\log(n-1)}{\log\left(\frac{1}{\lambda\gamma}\right)}\right\rceil\right]<1
\end{align}
where $\nu$ is the smallest eigenvalue of the Gram matrix $\Phi^T D_\mu\Phi$.
Then, for all $\delta$, with probability at least $1-\delta$, for all $n\geq n_0(\delta)$, $\hat A$ is invertible and we have:
\begin{align*} 
&\|v_{LSTD(\lambda)}-\hat v_{LSTD(\lambda)}\|_\mu \leq \frac{4V_{\text{max}}dL^2}{\sqrt{n-1}(1-\gamma)\nu}\sqrt{\left(1+\left\lceil\frac{\log(n-1)}{\log\left(\frac{1}{\lambda\gamma}\right)}\right\rceil\right)I(n-1,\delta)}+h(n,\delta)
\end{align*}
with $h(n,\delta)=\tilde O(\frac{1}{n})$.
\end{thm}
The constant $\nu$ is strictly positive under Assumption~\ref{assm:indfeat}. For all $\delta$, it is clear that the finite constant $n_0(\delta)$ exists since 
the l.h.s. of Equation~\eqref{caracn0} tends to $0$ when $n$ tends to infinity.
 As $\left(1+\left\lceil\frac{\log(n-1)}{\log\left(\frac{1}{\lambda\gamma}\right)}\right\rceil\right)I(n-1,\frac{\delta}{n^2})=\tilde O(1)$, we can see that LSTD($\lambda$) estimates $v_{LSTD(\lambda)}$ at the rate $\tilde O\left(\frac{1}{\sqrt n}\right)$. Finally, we can observe that since the function  $\lambda \mapsto \frac{1}{\log\left(\frac 1 {\lambda\gamma} \right)}$ is increasing, the rate of convergence deteriorates when $\lambda$ increases. 
This negative effect can be balanced by the fact that, as shown by the following result from the literature, the quality of  $v_{LSTD(\lambda)}$ improves when $\lambda$ increases.
\begin{thm}[\cite{vanroy:1997}]
\label{thm:errapp}
The approximation error satisfies\footnote{As suggested by V. Papavassilou \citep{vanroy:1997}, this bound can in fact be improved by using the Pythagorean theorem to $$\|v-v_{LSTD(\lambda)}\|_\mu  \leq \frac{1-\lambda\gamma}{\sqrt{(1-\gamma)(1+\gamma-2\lambda\gamma)}}\|v-\Pi v\|_\mu.$$ We keep the simple form of Theorem~\ref{thm:errapp} for simplicity.}:
\begin{align*}
\|v-v_{LSTD(\lambda)}\|_\mu  \le \frac{1-\lambda\gamma}{1-\gamma} \|v-\Pi v\|_\mu.
\end{align*}
\end{thm}
Since the constant equals $1$ when $\lambda=1$, one recovers the well-known fact that LSTD(1) computes the orthogonal projection $\Pi v$ of $v$.
By using the triangle inequality, one deduces from Theorems~\ref{thm:main} and \ref{thm:errapp} the following global error bound.
 \begin{cor}\label{cor:globalerror}
Let the assumptions and notations of Theorem~\ref{thm:main} hold. For all $\delta$, with probability $1-\delta$, for all $n\geq n_0(\delta)$, the global error of LSTD($\lambda$) satisfies:
\begin{align*} 
\|v-\hat v_{LSTD(\lambda)}\|_\mu \leq \frac{1-\lambda\gamma}{1-\gamma} \|v-\Pi v\|_\mu+\frac{4V_{\text{max}}dL^2}{\sqrt{n-1}(1-\gamma)\nu}
&\left(\left(\left\lceil\frac{\log(n-1)}{\log\left(\frac{1}{\lambda\gamma}\right)}\right\rceil+1\right)I(n-1,\delta)\right)^{\frac{1}{2}}+h(n,\delta).
\end{align*}
\end{cor}
\begin{remark}
The form of the result stated in Corollary~\ref{cor:globalerror} is slightly stronger than the one of \cite{lazaric:2012}:  for some property $P(n)$, our result if of the form ``$\forall \delta,~ \exists n_0(\delta),\text{such that}~ \forall n > n_0(\delta),~ P(n)$ holds with probability $1-\delta$'' while theirs is of the form ``$\forall n,~ \forall \delta,~ P(n)$ holds with probability $1-\delta$''.
Furthermore, under the same assumptions, the global error bound obtained by \cite{lazaric:2012}, in the restricted case where $\lambda=0$, has the following form:
\begin{align*}
\|\tilde v_{LSTD(0)}-v\|_\mu\leq \frac{4\sqrt{2}}{1-\gamma}\|v-\Pi v\|_\mu+ \tilde O\left(\frac{1}{\sqrt n}\right),
\end{align*}
where $\tilde v_{LSTD(0)}$ is the truncation (with $V_{\text{max}}$) of the pathwise LSTD solution\footnote{See \citep{lazaric:2012} for more details.}, while  we get in this analysis
\begin{align*}
\|\hat v_{LSTD(0)}-v\|_\mu\leq\frac{1}{1-\gamma}\|v-\Pi v\|_\mu+ \tilde O\left(\frac{1}{\sqrt n}\right).
\end{align*}
The term corresponding to the approximation error is a factor $4\sqrt{2}$ better with our analysis. Moreover, contrary to what we do here, the analysis of \cite{lazaric:2012} does not imply a rate of convergence for LSTD($\lambda$) (a bound on $\|v_{LSTD(0)}-\hat v_{LSTD(0)}\|_\mu$). Their arguments, based on a model of regression with Markov design, consists in \emph{directly} bounding the global error. Our two-step argument (bounding the estimation error with respect to $\|\cdot\|_\mu$, and then the approximation error with respect to $\|\cdot\|_\mu$) allows us to get a tighter result.
\end{remark}

As we have already mentioned, $\lambda=1$ minimizes the bound on the approximation error $\|v-v_{LSTD(\lambda)}\|$ (the first term in the r.h.s. in Corollary~\ref{cor:globalerror}) while $\lambda=0$ minimizes the bound on the estimation error $\|v_{LSTD(\lambda)}-\hat v_{LSTD(\lambda)}\|$ (the second term).
For any $n$, and for any $\delta$, there exists hence a value $\lambda^*$ that minimizes the global error bound by making an optimal compromise between the approximation and estimation errors. 
Figure~\ref{simulations} illustrates through simulations the interplay between $\lambda$ and $n$.
\begin{figure}
\begin{minipage}[c]{.49\linewidth}
\includegraphics[width=\linewidth]{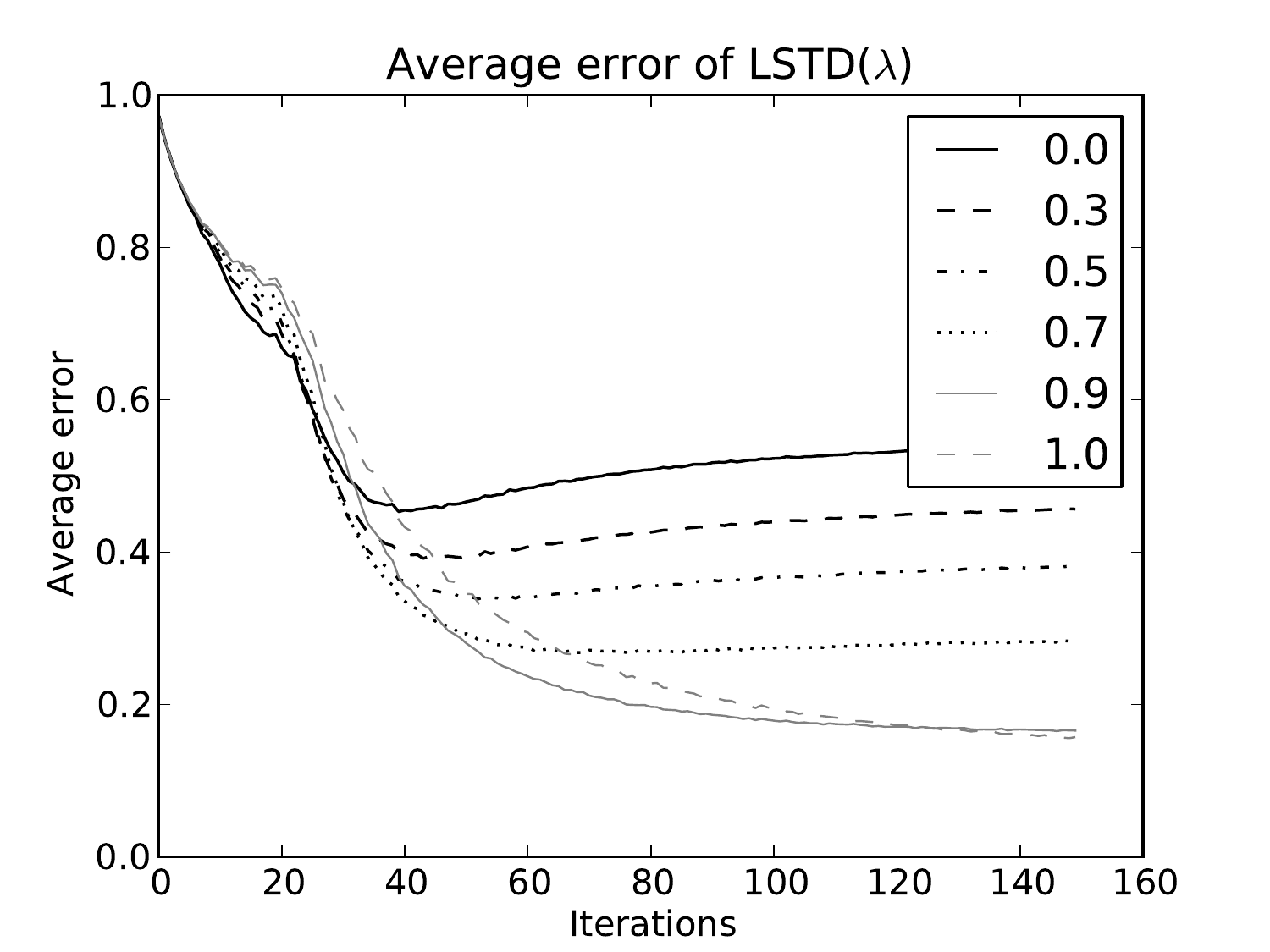}
\end{minipage}
\begin{minipage}[c]{.49\linewidth}
\includegraphics[width=\linewidth]{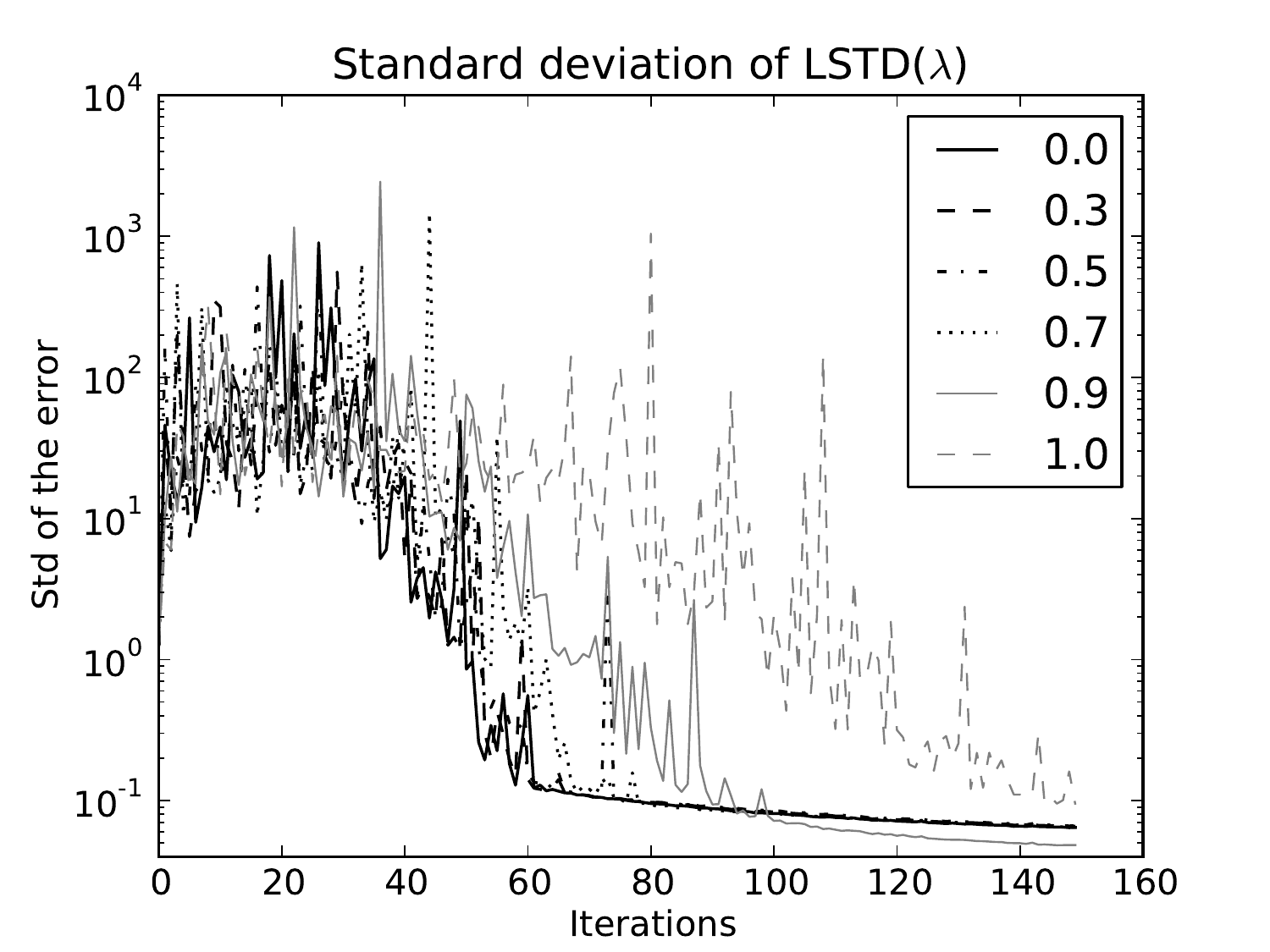}
\end{minipage}
\caption{\label{simulations}{\bf Learning curves for different values of $\lambda$}. We generated 1000 random Garnet MDPs \citep{Archibald:95} with $100$ states, random uniform rewards and $\gamma=0.99$. We also generated 1000 random feature spaces of dimension $20$ (by taking random matrices with random uniform entries). For all values of $\lambda \in \{0.0, 0.3, 0.5, 0.7, 0.9, 1.0\}$, we display (left) the average of the \emph{real} error and (right) the standard deviation with respect to the number of samples. Empirically, the best value of $\lambda$ appears to be a monotonic function of the number of samples $n$, that tends to~$1$ asymptotically. This is in accordance with our results in Corollary~\ref{cor:globalerror}.}
\end{figure}
The optimal value $\lambda^*$ depends on the process mixing parameters ($b$, $\kappa$ and $\overline{\beta}$) as well as on the quality of the policy space $\|v-\Pi v\|_\mu$, which are quantities that are usually unknown in practice. 
However, when the number of samples $n$ tends to infinity, it is clear that this optimal value $\lambda^*$ tends to $1$.

The next section contains a detailed proof of Theorem~\ref{thm:main}.

\section{Proof of Theorem~\ref{thm:main}}\label{section:proof}

In this section, we develop the arguments underlying the results of the previous section. The proof is organized in two parts. In a first preliminary part, we prove a concentration inequality for vector processes: a general result that  is based on infinitely-long eligibility traces. Then, in a second part, we actually prove Theorem~\ref{thm:main}: we apply this result to the error on estimating $A$ and $b$, and relate these errors with that on $v_{LSTD(\lambda)}$.

\subsection{Concentration inequality for infinitely-long trace-based estimates}
One of the first difficulties for the analysis of LSTD($\lambda$) is that the variables $A_i= z_i(\phi(X_i)-\gamma\phi(X_{i+1}))^T$ (respectively $b_i=z_ir (X_i)$) are not independent. Thus standard concentration results (like Lemma~\ref{lemma:hayes} we will describe in the Appendix~\ref{app:proofbetatrace}) for quantifying the speed at which the estimates converge to their limit cannot be used. As both terms $\hat A$ and $\hat b$ have the same structure, we will consider here a matrix that has the following general form:
\begin{align}
\label{eq:gform}
\hat G &= \frac{1}{n-1}\sum_{i=1}^{n-1} G_i \\
\mbox{with~~~}G_i &= z_i (\tau(X_i,X_{i+1}))^T
\end{align}
with $z_i$, defined in Equation~\eqref{eq:trace}, satisfies
$
z_i=\sum^{i}_{k=1}{(\lambda\gamma)^{i-k}\phi(X_k)}
$
and $\tau:\X^2 \mapsto \R^k$ is such that for $1\leq i\leq k$, $\tau_i$ belongs to $\B(\X^2,L')$ for some finite $L'$ \footnote{We denote $\X^i=\underbrace{X\times \X...\times \X}_{\text{i times}}$ for $i\geq 1$.}.  The variables $G_i$ are computed from one single trajectory, they are then significantly dependent. Nevertheless with the mixing assumption (Assumption \ref{assm:beta}), we can overcome this difficulty, and this by using a blocking technique due to \citet{Yu:1994}. This technique leads us back to the independent case. However the transition from the mixing case to the independent one requires stationarity (Lemma \ref{lemma:yu_decomp}) while  
$G_i$ as a $\sigma(\X^{i+1})$ measurable function  of the \emph{non-stationary} vector $(X_1,\dots,X_{i+1})$ does not define a \emph{stationary} process. In order to satisfy the \emph{stationarity} condition we will approximate $G_i$ by it truncated stationary version $G^m_i$. This is possible if we approximate 
$z_i$ by its $m$-truncated version:
$$
z_i^m=\sum^{i}_{k=\max(i-m+1,1)}{(\lambda\gamma)^{i-k}\phi(X_k)}.
$$
Since the function $\phi$ is bounded by some constant $L$ and the influence of the old events are controlled by some power of $\lambda \gamma <1$, it is easy to check that $\|z_i-z^m_i\|_\infty \le \frac{L}{1-\lambda\gamma} (\lambda \gamma)^m$. If we choose $m$ such that $m>\frac{\log (n-1)}{\log \frac{1}{\lambda\gamma}}$, we obtain $\|z_i-z^m_i\|_2 = O\left(\frac{1}{n}\right)$. 
Therefore it seems reasonable to approximate $\hat G$ with the process $\hat G^m$ satisfying
\begin{align}
\label{eq:gformtrunc}
\hat G^m &= \frac{1}{n-1}\sum_{i=1}^{n-1} G_i^m, \\
\mbox{with~~~}G^m_i &= z_i^m (\tau(X_i,X_{i+1}))^T.
\end{align}
For all $i \ge m$, $G_i^m$ is a $\sigma(\X^{m+1})$ measurable function of the \emph{stationary} vector $Z_i=(X_{i-m+1}, X_{i-m+2}$ $,\dots, X_{i+1})$. So we can apply the blocking technique of \citet{Yu:1994} to $G^m_i$, but before to do so we have to check out whether $G^m_i$ well defines a $\beta$-mixing process. It can be shown \citep{Yu:1994} 
that  any mesurable function $f$ of a $\beta$-mixing process is a $\beta^f$-mixing process with $\beta^f\leq \beta$, so we only have to prove that the process $Z_i$ is a $\beta$-mixing process. For that we need to  relate its $\beta$ coefficients to those of $(X_i)_{i \ge 1}$ on which Assumption~\ref{assm:beta} is made. This is the purpose of the following Lemma.
\begin{lemma}\label{lemma:beta_relation}
Let $(X_n)_{n\geq 1}$ be a $\beta$-mixing process, then $(Z_n)_{n\geq 1}=(X_{n-m+1}, X_{n-m+2}$ $,\dots, X_{n+1})_{n\geq 1}$ is a $\beta$-mixing process such that its $i^{th}$ $\beta$ mixing coefficient $\beta^Z_i$ satisfies $\beta^Z_i\leq\beta^X_{i-m}$.
\end{lemma}
\begin{proof}
Let $\Gamma=\sigma(Z_1,...,Z_t)$, by definition we have
\begin{align*}
\Gamma=\sigma(Z^{-1}_j(B):j\in \{1,...,t\},B\in \sigma(\X^{m+1})). 
\end{align*}
 For all $j\in \{1,...,t\}$ we have 
\begin{align*}
Z^{-1}_j(B)=\left\{\omega\in \Omega, Z_j(\omega)\in B\right\}.
\end{align*}
For $B=B_0\times...\times B_m$, we observe that
\begin{align*}
Z^{-1}_j(B)=\{\omega\in \Omega, X_j(\omega)\in B_0,...,X_{j+m}(\omega)\in B_m\}.
\end{align*}
Then we have 
\begin{align*}
\Gamma=\sigma(X^{-1}_j(B):j\in \{1,...,t+m\},B\in \sigma(\X))=\sigma(X_1,...,X_{t+m}).
\end{align*}
Similarly we can prove that $\sigma(Z^\infty_{t+i})=\sigma(X^\infty_{t+i})$. Then let $\beta^X_i$ be the $i^{th}$ $\beta$-mixing coefficient of the process $(X_n)_{n\geq 1}$, we have
\begin{align*}
\beta^X_i=\sup_{t\geq 1}\E\left[\sup_{B\in \sigma(X^\infty_{t+i})}\left|P(B|\sigma(X_1,...,X_t))-P(B)\right|\right].
\end{align*}
Similarly for the process $(Z_n)_{n\geq 1}$ we can see that
\begin{align*}
\beta^Z_i=\sup_{t\geq 1}\E\left[\sup_{B\in \sigma(Z^\infty_{t+i})}\left|P(B|\sigma(Z_1,...,Z_t))-P(B)\right|\right].
\end{align*}
By applying what we developped above we obtain
\begin{align*}
\beta^Z_i=\sup_{t\geq 1}\E\left[\sup_{B\in \sigma(X^\infty_{t+i})}\left|P(B|\sigma(X_1,...,X_{t+m}))-P(B)\right|\right].
\end{align*}
Denote $t'=t+m$ then for $i>m$ we have
\begin{align*}
\beta^Z_i&=\sup_{t'\geq m+1}\E\left[\sup_{B\in \sigma(X^\infty_{t'+i-m})}\left|P(B|\sigma(X_1,...,X_{t'}))-P(B)\right|\right]\\
&\leq \beta^X_{i-m}.
\end{align*}
~\vspace{-1.1cm}

\end{proof}
Let $\|.\|_F$ denote the Frobenius norm satisfying : for $M\in \R^{d\times k}$, $\|M\|^2_F=\sum^d_{l=1}\sum^k_{j=1} (M_{l,j})^2$. We are now ready to prove the concentration inequality for the infinitely-long-trace $\beta$-mixing process $\hat G$. 
\begin{lemma}\label{lemma:concentration_matrix}
Let Assumptions~\ref{assm:indfeat} and \ref{assm:beta} hold and let $X_1 \sim \mu$. 
Define the $d\times k$ matrix $G_i$ such that
\begin{align}
G_i=\sum^{i}_{k=1}{(\lambda\gamma)^{i-k}\phi(X_k)} (\tau(X_{i},X_{i+1}))^T.
\end{align}
Recall that $\phi=(\phi_1,\dots,\phi_d)$ is such that for all $j$, $\phi_j \in \B(\X,L)$, and that $\tau \in \B(\X^2,L')$. Then for all $\delta$ in $(0,1)$, with probability $1-\delta$,
\begin{align*}
\left\| \frac{1}{n-1} \sum_{i=1}^{n-1}{G_i} - \frac{1}{n-1}\sum_{i=1}^{n-1}{\E[G_i]}\right\|_2 \le \frac{2\sqrt{d\times k}LL'}{(1-\lambda\gamma)\sqrt{n-1}}\sqrt{\left(\left\lceil\frac{\log(n-1)}{\log\left(\frac{1}{\lambda\gamma}\right)}\right\rceil+1\right)J(n-1,\delta)} + \epsilon(n),
\end{align*}
where  
\begin{align*}
J(n,\delta)&=32\Gamma(n,\delta)\max\left\{\frac{\Gamma(n,\delta)}{b},1\right\}^{\frac{1}{\kappa}},\\
\Gamma(n,\delta)&=\log\left(\frac{2}{\delta}\right)+\log(\max\{4e^2,n\overline\beta\}),\\
\epsilon(n)&=2\left\lceil\frac{\log(n-1)}{\log\left(\frac{1}{\lambda\gamma}\right)}\right\rceil\frac{\sqrt{d\times k}LL'}{(n-1)(1-\lambda\gamma)}.
\end{align*}
\end{lemma}
Note that with respect to the quantities $I$ and $\Lambda$ introduced in Theorem~\ref{thm:main}, the quantities we introduce here
are such that $J(n,\delta)=I(n,4n^2 \delta)$ and $\Gamma(n,\delta)=\Lambda(n,4n^2 \delta)$.
\begin{proof}
The proof amounts to show that i) the approximation due to considering the estimate $\hat G^m$ with truncated traces instead of $\hat G$ is bounded by $\epsilon(n)$, and then ii) to apply the block technique of \citet{Yu:1994} in a way somewhat similar to---but technically slightly more involved than---what \cite{lazaric:2012} did for LSTD(0). We defer the technical arguments to Appendix~\ref{app:proofbetatrace} for readability.
\end{proof}
Using a very similar proof, we can derive a (simpler) general concentration inequality for $\beta$-mixing processes:
\begin{lemma}\label{lemma:1}
\label{concvec}
Let  $Y=(Y_1,\dots,Y_n)$ be random variables taking their values in the space ${\R}^d$, generated from a stationary exponentially $\beta$-mixing process 
with parameters  $\overline \beta$, $b$ and $\kappa$, and such that for all  $i$, $\|Y_i-\E[Y_i]\|_2 \le B_2$ almost surely. 
Then for all $\delta>0$,
\begin{align*}
\P \left\{\left\| \frac{1}{n} \sum_{i=1}^{n}{Y_i}-\frac{1}{n} \sum_{i=1}^{n}{\E[Y_i]} \right\|_2 \le \frac{B_2}{\sqrt{n}} \sqrt{J(n,\delta)} \right\}> 1-\delta
\end{align*}
where $J(n,\delta)$ is defined as in Lemma~\ref{lemma:concentration_matrix}.
\end{lemma}
\begin{remark}
If the variables $Y_i$ were independent, we would have $\beta_i=0$ for all $i$, that is we could choose $\overline \beta=0$ and $b=\infty$, so that $J(n,\delta)$ reduces to $32 \log\frac{8 e^2}{\delta}=O(1)$ and we recover standard results such as the one we describe in Lemma~\ref{lemma:hayes} we will describe in the Appendix~\ref{app:proofbetatrace}. Furthermore, the price to pay for having a $\beta$-mixing assumption (instead of simple independence) lies in the extra coefficient $J(n,\delta)$ which is $\tilde O(1)$; in other words, it is rather mild. 
\end{remark}

\subsection{Proof of Theorem~\ref{thm:main}}

After having introduced the corresponding concentration inequality for infinitely-long trace-based estimates we are ready to prove Theorem~\ref{thm:main}. The first important step to Theorem~\ref{thm:main} proof consists in deriving the following lemma.
\begin{lemma}\label{lemma:estimation_error}
Write $\epsilon_A=\hat A-A$, $\epsilon_b=\hat b-b$ and $\nu$ the smallest eigenvalue of the matrix $\Phi^T D_\mu \Phi$. For all $\lambda\in (0,1)$, the estimate $\hat v_{LSTD(\lambda)}$ satisfies\footnote{When $\hat A$ is not invertible, we take $\hat v_{LSTD(\lambda)}=\infty$ and the inequality is always satisfied since, as we will see shortly,  the invertiblity of $\hat A$ is equivalent to that of $(I+\epsilon_AA^{-1}).$}:
\begin{equation*}
\|v_{LSTD(\lambda)}-\hat v_{LSTD(\lambda)}\|_\mu\leq \frac{1-\lambda\gamma}{(1-\gamma)\sqrt{\nu}}\|(I+\epsilon_AA^{-1})^{-1}\|_2\|\epsilon_A\theta-\epsilon_b\|_2 \label{maineq}.
\end{equation*}
Furthermore,  if for some $\epsilon$ and $C$,  $\|\epsilon_A\|_2 \le \epsilon < C \le \frac{1}{\|A^{-1}\|_2}$, then $\hat A$ is invertible and 
\begin{equation*}
\|(I+\epsilon_AA^{-1})^{-1}\|_2 \le \frac{1}{1-\frac{\epsilon}{C}}.
\end{equation*}
\end{lemma}
\begin{proof}
Starting from the definitions of $v_{LSTD(\lambda)}$ and $\hat v_{LSTD(\lambda)}$, we have 
\begin{align}
\hat v_{LSTD(\lambda)}-v_{LSTD(\lambda)}&=\Phi \hat \theta - \Phi \theta \nonumber \\
& = \Phi A^{-1}(A\hat\theta-b)\label{eq:difference}.
\end{align}
On the one hand, with the expression of $A$ in Equation~\eqref{eq:defA}, writing $M=(1-\lambda)\gamma P(I-\lambda\gamma P)^{-1}$ and $M_\mu=\Phi^TD_\mu \Phi$, and using some linear algebra arguments, we can observe that
\begin{align*} 
\Phi A^{-1}&= \Phi \left[ \Phi^TD_\mu(I-\gamma P)(I-\lambda\gamma P)^{-1}\Phi\right]^{-1} \nonumber\\
&=\Phi   \left[ \Phi^TD_\mu (I-\lambda\gamma P-(1-\lambda)\gamma P)(I-\lambda\gamma P)^{-1}\Phi \right]^{-1}\nonumber\\
&=\Phi  (M_\mu-\Phi^TD_\mu M\Phi)^{-1}. \nonumber
\end{align*}
Since the matrices $A$ and $M_\mu$ are invertible, the matrix $(I-M^{-1}_\mu\Phi^TD_\mu M\Phi)$ is also invertible, then  
\begin{align*}
\Phi A^{-1}=\Phi (I-M^{-1}_\mu\Phi^TD_\mu M\Phi)^{-1}M^{-1}_\mu. 
\end{align*}
We know from \cite{vanroy:1997} that $\|\Pi\|_\mu=1$---the projection matrix $\Pi$ is defined in~Equation~\eqref{eq:defPi}---and $\| P\|_\mu=1$. Hence,  we have $\|\Pi M\|_\mu=\frac{(1-\lambda)\gamma}{1-\lambda\gamma} <1$ and the matrix $(I-\Pi M)$ is invertible. We can use the identity $X(I-YX)^{-1}=(I-XY)^{-1}X$ with $X=\Phi$ and $Y=M_\mu^{-1}\Phi^T D_\mu M$, and obtain
\begin{align}
\Phi A^{-1} = (I-\Pi M)^{-1} \Phi M^{-1}_\mu \label{eq:defAinv}.
\end{align}
On the other hand, using the facts that $A\theta=b$ and $\hat A \hat \theta= \hat b$, we can see that:
\begin{align}
A \hat \theta-b &= A \hat \theta - b - (\hat A \hat \theta - \hat b) \nonumber\\
& = \hat b - b - \epsilon_A \hat  \theta \nonumber\\
& = \hat b - b - \epsilon_A \theta +  \epsilon_A \theta - \epsilon_A \hat \theta \nonumber\\
& = \hat b - b - (\hat A - A) \theta + \epsilon_A ( \theta-\hat \theta) \nonumber\\
& = \hat b - \hat A \theta - (b-A\theta) + \epsilon_A A^{-1} ( A \theta - A \hat \theta) \nonumber\\
& =  \hat b - \hat A \theta + \epsilon_A A^{-1} ( b - A \hat \theta). \nonumber
\end{align}
Then we have 
\begin{align*}
A \hat \theta-b=\hat b - \hat A \theta -\epsilon_A A^{-1} (b -A \hat \theta).
\end{align*}
Consequently
\begin{align}
A \hat \theta-b&= (I+\epsilon_AA^{-1})^{-1}(\hat b-\hat A \theta)\nonumber\\
&= (I+\epsilon_AA^{-1})^{-1}(\epsilon_b-\epsilon_A \theta)\label{phia2}
\end{align}
 where the last equality follows from the identity $A\theta=b$. Using Equations~\eqref{eq:defAinv} and \eqref{phia2}, Equation~\eqref{eq:difference} can be rewritten as follows:
\begin{align}
\hat v_{LSTD(\lambda)}-v_{LSTD(\lambda)}&= (I-\Pi M)^{-1} \Phi M^{-1}_\mu (I+\epsilon_AA^{-1})^{-1}(\epsilon_b-\epsilon_A \theta). \label{eq:compwise}
\end{align}
Now we will try to bound $\|\Phi M^{-1}_\mu (I+\epsilon_AA^{-1})^{-1}(\epsilon_b-\epsilon_A \theta)\|_\mu$. Notice that for all $x$, 
\begin{align}
\|\Phi M_\mu^{-1} x\|_{\mu}& =\sqrt{x^T M_\mu^{-1} \Phi^T D_\mu \Phi M_\mu^{-1} x}\label{eq:matrix_prop} 
= \sqrt{x^T M_\mu^{-1} x} 
 \le \frac{1}{\sqrt{\nu}} \|x\|_2
\end{align}
where $\nu$ is the smallest (real) eigenvalue of the Gram matrix $M_\mu$. By taking the norm in Equation~\eqref{eq:compwise} and using the above relation, we get
\begin{align*}
\| \hat v_{LSTD(\lambda)}-v_{LSTD(\lambda)} \|_\mu & \le \|(I-\Pi M)^{-1}\|_\mu \|\Phi M^{-1}_\mu (I+\epsilon_AA^{-1})^{-1} (\epsilon_b-\epsilon_A \theta) \|_\mu \\
& \le \|(I-\Pi M)^{-1}\|_\mu \frac{1}{\sqrt{\nu}} \| (I+\epsilon_AA^{-1})^{-1} (\epsilon_A\theta -\epsilon_b)\|_2 \\
& \le \|(I-\Pi M)^{-1}\|_\mu \frac{1}{\sqrt{\nu}} \| (I+\epsilon_AA^{-1})^{-1} \|_2 \|\epsilon_A\theta -\epsilon_b\|_2.
\end{align*}
The first part of the lemma is obtained by using the fact that $\|\Pi M\|_\mu=\frac{(1-\lambda)\gamma}{1-\lambda\gamma}<1$, which imply that
\begin{align}
\|(I-\Pi M)^{-1}\|_\mu= \left\| \sum^{\infty}_{i=0}{(\Pi M)^i} \right\|_\mu \leq \sum^\infty_{i=0}{\|\Pi M\|^i_\mu}\leq \frac{1}{1-\frac{(1-\lambda)\gamma}{1-\lambda\gamma}} = \frac{1-\lambda\gamma}{1-\gamma}. \label{eq:ipm}
\end{align}

We are going now to prove the second part of the Lemma. Since $A$ is invertible, the matrix $\hat A$ is invertible if and only if the matrix $\hat AA^{-1}=(A+\epsilon_A)A^{-1}=I+\epsilon_AA^{-1}$ is invertible. Let us denote $\rho(\epsilon_AA^{-1})$ the spectral radius of the matrix $\epsilon_AA^{-1}$. A sufficient condition for $\hat AA^{-1}$ to be invertible is that  $\rho(\epsilon_AA^{-1})<1$. From the inequality $\rho(M)\leq \|M\|_2$ for any square matrix $M$, we can see that for any $C$ and $\epsilon$ that satisfy
$\|\epsilon_A\|_2\leq \epsilon<C<\frac{1}{\|A^{-1}\|_2}$, we have
\begin{align*}
\rho(\epsilon_AA^{-1})\leq \|\epsilon_AA^{-1}\|_2\leq \|\epsilon_A\|_2\|A^{-1}\|_2\leq \frac{\epsilon}{C}<1.
\end{align*}
It follows that the matrix $\hat A$ is invertible and 
\begin{align*}
\|(I+\epsilon_AA^{-1})^{-1}\|_2&= \left\| \sum_{i=0}^\infty (\epsilon_AA^{-1})^i \right\|_2
 \le  \sum_{i=0}^\infty \left(\frac{\epsilon}{C}\right)^i 
= \frac{1}{1-\frac{\epsilon}{C}}.
\end{align*}
This concludes the proof of Lemma~\ref{lemma:estimation_error}.  
\end{proof}
To finish the proof of Theorem~\ref{thm:main}, Lemma~\ref{lemma:estimation_error} suggests that we should control both terms $\|\epsilon_A\|_2$ and $\|\epsilon_A \theta-\epsilon_b\|_2$ with high probability. This is what we do now. 
\paragraph{Controlling $\|\epsilon_A\|_2$.}
By the triangle inequality, we can see that
\begin{align}
\label{mp1}
\|\epsilon_A\|_2 & \le \| \E[\epsilon_A] \|_2 + \| \epsilon_A-\E[\epsilon_A] \|_2.
\end{align}
Write $\hat A_{n,k}=\phi(X_k) (\phi(X_n)-\gamma\phi(X_{n+1}))^T$. For all $n$ and $k$, we have $\|\hat A_{n,k}\|_2 \le 2dL^2$.  We can bound the first term of the r.h.s. of Equation~\eqref{mp1} as follows, by replacing $A$ with its expression in~\eqref{eq:defA}:
\begin{align*}
\|\E[\epsilon_A]\|_2 &=  \left\| A - \E \left[ \frac{1}{n-1} \sum_{i=1}^{n-1} \sum^i_{k=1}(\lambda\gamma)^{i-k} \hat A_{i,k} \right] \right\|_2 \\
& = \left\| \E \left[ \frac{1}{n-1} \sum_{i=1}^{n-1} \left( \sum_{k=-\infty}^i (\lambda\gamma)^{i-k} \hat A_{i,k} - \sum^i_{k=1}(\lambda\gamma)^{i-k} \hat A_{i,k} \right) \right] \right\|_2 \\
& = \left\| \E \left[ \frac{1}{n-1} \sum_{i=1}^{n-1} (\lambda\gamma)^i \sum_{k=-\infty}^0 (\lambda\gamma)^{-k} \hat A_{i,k} \right] \right\|_2 \\
& \le \frac{1}{n-1} \sum_{i=1}^{n-1} (\lambda\gamma)^i \frac{2dL^2}{1-\lambda\gamma} \\
& \le \frac{1}{n-1}\frac{2dL^2}{(1-\lambda\gamma)^2} = \epsilon_0(n).
\end{align*}
Let $(\delta_n)$ a parameter in $(0,1)$ depending on $n$, that we will fix later, a consequence of Equation~\eqref{mp1} and the just derived bound is that:
\begin{align*}
\P\left\{\|\epsilon_A\|_2 \ge \epsilon_1(n,\delta_n)\right\} & \le \P \{ \|\epsilon_A - \E[\epsilon_A] \|_2 \ge \epsilon_1(n,\delta_n)-\epsilon_0(n) \}\\
& \le \delta_n
\end{align*}
if we choose $\epsilon_1(n,\delta_n)$ such that (cf. Lemma~\ref{lemma:concentration_matrix}) 
\[\epsilon_1(n,\delta_n)-\epsilon_0(n)  = \frac{4dL^2}{(1-\lambda\gamma)\sqrt{n-1}}\sqrt{\left(\left\lceil\frac{\log(n-1)}{\log\left(\frac{1}{\lambda\gamma}\right)}\right\rceil+1\right) J(n-1,\delta_n)}+\epsilon(n)\] where $\epsilon(n)=\frac{4mdL^2}{(n-1)(1-\lambda\gamma)}$, that is if 
\begin{align}
\label{eq:defeps1}
\epsilon_1(n,\delta_n) &=  \frac{4dL^2}{(1-\lambda\gamma)\sqrt{n-1}} \sqrt{\left(\left\lceil\frac{\log(n-1)}{\log\left(\frac{1}{\lambda\gamma}\right)}\right\rceil+1\right) J(n-1,\delta_n)}+\epsilon(n)+\epsilon_0(n).
\end{align}


\paragraph{Controlling $\|\epsilon_A\theta -\epsilon_b\|_2$.}
By using the fact that $A \theta = b$, the definitions of $\hat A$ and $\hat b$, and the fact that $\phi(x)^T \theta=[\phi \theta](x)$, we have
\begin{align*}
\epsilon_A\theta -\epsilon_b &= \hat A \theta - \hat b \\
& = \frac{1}{n-1}\sum^{n-1}_{i=1}z_i(\phi(X_i)-\gamma \phi(X_{i+1})^T)\theta -  \frac{1}{n-1}\sum^{n-1}_{i=1}z_i r(X_i) \\
& = \frac{1}{n-1}\sum^{n-1}_{i=1}z_i([\phi \theta](X_i)-\gamma [\phi \theta](X_{i+1})^T  - r(X_i)) \\
& = \frac{1}{n-1}\sum^{n-1}_{i=1}z_i\Delta_i
\end{align*}
where, since $v_{LSTD(\lambda)}=\Phi \theta$, $\Delta_i$ is the following number:
\[\Delta_i=v_{LSTD(\lambda)}(X_i)-\gamma v_{LSTD(\lambda)}(X_{i+1})-r(X_i).\]
We can control $\|\epsilon_A\theta -\epsilon_b\|_2$ by following the same proof steps as above. In fact we have
\begin{align}
\|\epsilon_A\theta -\epsilon_b\|_2&\leq \|\epsilon_A\theta -\epsilon_b-\E[\epsilon_A\theta -\epsilon_b]\|_2+\|\E[\epsilon_A\theta -\epsilon_b]\|_2,\label{eq:mp2}\\
\text{and}~~\|\E[\epsilon_A\theta -\epsilon_b]\|_2&\leq \|\E[\epsilon_A]\|_2\|\theta\|_2+\|\E[\epsilon_b]\|_2.\nonumber
\end{align}
From what have been developed before we can see that $\|\E[\epsilon_A]\|_2\leq \epsilon_0(n)=\frac{1}{n-1}\frac{2dL^2}{(1-\lambda\gamma)^2}$.
Similarly we can show that $\|\E[\epsilon_b]\|_2\leq \frac{1}{n-1}\frac{\sqrt{d}LR_{\text{max}}}{(1-\lambda\gamma)^2}$. We can hence conclude that
\begin{align*}
\|\E[\epsilon_A\theta -\epsilon_b]\|_2\leq \frac{1}{n-1}\frac{2dL^2}{(1-\lambda\gamma)^2}\|\theta\|_2+\frac{1}{n-1}\frac{\sqrt{d}LR_{\text{max}}}{(1-\lambda\gamma)^2}=\epsilon'_0(n).
\end{align*}
As a consequence of Equation~\eqref{eq:mp2} and the just derived bound we have 
\begin{align*}
\P(\|\epsilon_A\theta-\epsilon_b\|_2\geq \epsilon_2(\delta_n))\leq\P(\|\epsilon_A\theta-\epsilon_b-\E[\epsilon_A\theta-\epsilon_b]\|_2\geq \epsilon_2(\delta_n)-\epsilon'_0(n))\leq\delta_n
\end{align*}
if we choose $\epsilon_2(\delta_n)$ such that (cf Lemma~\ref{lemma:concentration_matrix})
\begin{align}
\epsilon_2(\delta_n)=\frac{2\sqrt dL\|\Delta_i\|_\infty}{(1-\lambda\gamma)\sqrt{n-1}} \sqrt{\left(\left\lceil\frac{\log(n-1)}{\log\left(\frac{1}{\lambda\gamma}\right)}\right\rceil+1\right) J(n-1,\delta_n)}+\frac{2\sqrt{d}L  \|\Delta_i\|_\infty}{(n-1)(1-\lambda\gamma)}\left\lceil\frac{\log(n-1)}{\log\left(\frac{1}{\lambda\gamma}\right)}\right\rceil+\epsilon'_0(n).\label{eq:defeps2}
\end{align}

It remains to compute a bound on $\|\Delta_i\|_\infty$. To do so,  it suffices
to bound $v_{LSTD(\lambda)}$. For all $x\in \X$, we have
\begin{align*}
|v_{LSTD(\lambda)}(x)|=|\phi^T(x)\theta|\leq \|\phi^T(x)\|_2\|\theta\|_2\leq\sqrt{d}L\|\theta\|_2,
\end{align*}
where the first inequality is obtained from the Cauchy-Schwarz inequality. We thus need to bound $\|\theta\|_2$. On the one hand,  we have
\begin{align*}
\|v_{LSTD(\lambda)}\|_\mu=\|\Phi \theta\|_\mu = \sqrt{\theta^T M_\mu \theta} \geq \sqrt \nu \|\theta\|_2,
\end{align*}
and on the other hand, we have
\begin{align*}
\|v_{LSTD(\lambda)}\|_\mu=\|(I-\Pi M)^{-1}\Pi (I-\lambda\gamma P)^{-1}r\|_\mu\leq \frac{R_{\text{max}}}{1-\gamma}=V_{\text{max}}.
\end{align*}
Therefore 
\begin{align*}
\|\theta \|_2\leq \frac{V_{\text{max}}}{\sqrt \nu}.
\end{align*}
We can conclude that 
\begin{align*}
\forall x\in \X,~|v_{LSTD(\lambda)}(x)|\leq\frac{\sqrt{d}LV_{\text{max}}}{\sqrt \nu}.
\end{align*}
 Then for all $i$ we have 
\begin{align*}
|\Delta_i|&=|v_{LSTD(\lambda)}(X_i)-\gamma v_{LSTD(\lambda)}(X_{i+1})-r(X_i)|\\
&\leq \frac{\sqrt{d}LV_{\text{max}}}{\sqrt \nu}+\gamma\frac{\sqrt{d}LV_{\text{max}}}{\sqrt \nu}+ (1-\gamma)V_{\text{max}}.
\end{align*}
Since $\Phi^TD_\mu\Phi$ is a symmetric matrix, we have $\nu\leq\|\Phi^TD_\mu\Phi\|_2$. We can see that
\begin{align*}
\|\Phi^TD_\mu\Phi\|_2\leq d\max_{j,k}|\phi^t_kD_\mu\phi_j|=d\max_{j,k}|\phi^t_kD^{\frac{1}{2}}_\mu D^{\frac{1}{2}}_\mu\phi_j|
\leq  d\max_{j,k}\|\phi^t_k\|_\mu\|\phi_j\|_\mu\leq dL^2,
\end{align*}
so that $\nu\leq dL^2$. It follows that, for all $i$  
\begin{align*}
|\Delta_i|\leq\frac{\sqrt{d}LV_{\text{max}}}{\sqrt \nu}+\gamma\frac{\sqrt{d}LV_{\text{max}}}{\sqrt \nu}+ \frac{\sqrt dL}{\sqrt{\nu}}(1-\gamma)V_{\text{max}}=2\frac{\sqrt dL}{\sqrt{\nu}}V_{\text{max}}.
\end{align*}


\paragraph{Conclusion of the proof.}
We are ready to conclude the proof.
Now that we know how to control both terms $\|\epsilon_A\|_2$ and $\|\epsilon_A\theta-\epsilon_b\|_2$, we can see that
\begin{align*}
& \P\left\{ \exists n \ge 1, ~\{ \|\epsilon_A\|_2 \ge \epsilon_1(n,\delta_n) \} \cup \{ \|\epsilon_A\theta-\epsilon_b)\|_2 \ge \epsilon_2(n,\delta_n) \}\right\} \\
\le &\sum_{n=1}^\infty \P\left\{  \|\epsilon_A\|_2 \ge \epsilon_1(n,\delta_n) \right\} + P\left\{ \|\epsilon_A\theta-\epsilon_b)\|_2 \ge \epsilon_2(n,\delta_n) \right\}\\
\le &2 \sum_{n=1}^{\infty} \delta_n =  \frac{1}{2} \frac{\pi^2}{6} \delta < \delta
\end{align*}
if we choose $\delta_n=\frac{1}{4 n^2}\delta$.
By the second part of Lemma~\ref{lemma:estimation_error}, for all $\delta$, with probability at least $1-\delta$, for all $n$ such that $\epsilon_1(n,\delta_n)<C$, $\hat A$ is invertible and
\begin{align*}
\|v_{LSTD(\lambda)}-\hat v_{LSTD(\lambda)}\|_\mu &\le \frac{1-\lambda\gamma}{(1-\gamma)\sqrt{\nu}} \frac{\epsilon_2\left(n,\delta_n\right)}{1-\frac{\epsilon_1\left(n,\delta_n\right)}{C}} \\
 & = \frac{1-\lambda\gamma}{(1-\gamma)\sqrt{\nu}} \left[ \epsilon_2\left(n,\delta_n\right) + \frac{\epsilon_1\left(n,\delta_n\right)\epsilon_2\left(n,\delta_n\right)}{C-{\epsilon_1\left(n,\delta_n\right)}} \right].
\end{align*}
We get the bound of the Theorem by replacing $\epsilon_1(n,\delta_n)$ and $\epsilon_2(n,\delta_n)$ with their definitions in Equations~\eqref{eq:defeps1} and~\eqref{eq:defeps2}.

 To complete the proof of Theorem~\ref{thm:main}, we now need to show how to pick $C$, which will allow to show that the condition $\epsilon_1(n,\delta_n)<C\leq\frac{1}{\|A^{-1}\|_2}$ is equivalent to the one that characterizes the index $n_0(\delta)$ in the Theorem. Indeed we have
\begin{align*}
\forall v\in \R^d, \|\Phi A^{-1}v\|_\mu = \sqrt{(A^{-1}v)^T M_\mu A^{-1}v} \geq \sqrt \nu \|A^{-1}v\|_2.
\end{align*}
We know that
\begin{align*}
\|\Phi A^{-1}v\|_\mu=\|(I-\Pi M)^{-1}\Phi M^{-1}_\mu v\|_\mu\leq \frac{1-\lambda \gamma}{1-\gamma}\|\Phi M^{-1}_\mu v\|_\mu\leq \frac{1-\lambda \gamma}{(1-\gamma)\sqrt \nu}\|v\|_2
\end{align*}
where the last inequality is obtained from Equation~\eqref{eq:matrix_prop}.
Then \[\|A^{-1}\|_2\leq \frac{1-\lambda \gamma}{(1-\gamma)\nu},\]
and consequently we can take $C=\frac{(1-\gamma)\nu}{1-\lambda\gamma}$. This concludes the proof of Theorem~\ref{thm:main}.



\section{Conclusion and Future Work}\label{section:4}

This paper introduces a high-probability convergence rate for the algorithm LSTD($\lambda$) in terms of the number of samples $n$ and the parameter $\lambda$. We have shown that this convergence is at the rate of $\tilde O(\frac{1}{\sqrt{n}})$, in the case where the samples are generated from a stationary $\beta$-mixing process.
To do so, we  introduced an original vector concentration inequality (Lemma~\ref{lemma:concentration_matrix}) for estimates that are based on eligibility traces. 
A simplified version of this concentration inequality (Lemma~\ref{lemma:1}), that applies to general stationary beta-mixing processes,  may be useful in many other contexts where we want to relax the i.i.d. hypothesis on the samples. 

The performance bound that we deduced is more accurate than the one from \cite{lazaric:2012}, restricted to the case $\lambda=0$.
The analysis that they proposed was based on a Markov design regression model. 
By using the trace truncation technique we have employed, we believe it is possible to extend the proof of \cite{lazaric:2012} to the general case $\lambda$ in $(0,1)$.
However we would still pay a $4\sqrt{2}$ extra factor in the final bound.

In the future, we plan to instantiate our new bound in a Policy Iteration context like \citet{lazaric:2012} did for LSTD(0). An interesting follow-up work would also be to extend our analysis of LSTD($\lambda$) to the situation where one considers non-stationary policies, as \citet{scherrer:2012} showed that it allows to improve the overall performance of the Policy Iteration Scheme. Finally, a challenging question would be to consider LSTD($\lambda$) in the off-policy case, for which the convergence has recently been proved by~\citet{yu:2010}.


\appendix


\section{Proof of Lemma~\ref{lemma:concentration_matrix}}
\label{app:proofbetatrace}

Writing for a given integer $m>1$
\begin{align*}
\epsilon_1 & =\frac{1}{n-1} \sum_{i=1}^{m-1}{G_i-\E[G_i]} \\
\mbox{and~~~}\epsilon_2 & = \frac{1}{n-1} \sum_{i=m}^{n-1}(z_i-z_i^m)\tau(X_i,X_{i+1})^T-\E[(z_i-z_i^m)\tau(X_i,X_{i+1})^T],
\end{align*}
we have
\begin{align}
\frac{1}{n-1} \sum_{i=1}^{n-1}{G_i-\E[G_i]} & = \frac{1}{n-1} \sum_{i=m}^{n-1}{G_i-\E[G_i]} + \epsilon_1 \nonumber \\
& = \frac{1}{n-1} \sum_{i=m}^{n-1}z_i\tau(X_i,X_{i+1})^T-\E[z_i\tau(X_i,X_{i+1})^T] + \epsilon_1 \nonumber \\
& = \frac{1}{n-1} \sum_{i=m}^{n-1}z_i^m \tau(X_i,X_{i+1})^T-\E[z_i^m \tau(X_i,X_{i+1})^T] + \epsilon_1+\epsilon_2\nonumber\\
&=\frac{1}{n-1}\sum^{n-1}_{i=m}(G^m_i-\E[G^m_i])+ \epsilon_1+\epsilon_2. \label{eq:bt1}
\end{align}
For all $i$, we have $\| z_i \|_\infty \le \frac{L}{1-\lambda\gamma}$, $\|G_i\|_\infty \le \frac{LL'}{1-\lambda\gamma}$, and $\| z_i-z_i^m \|_\infty \le \frac{(\lambda\gamma)^{m} L}{1-\lambda\gamma}$. As a consequence---using $\|M\|_2 \le \|M\|_F=\sqrt{d\times k}\|x\|_\infty$ for $M \in \R^{d\times k}$ with $x$ the vector obtained by concatenating all $M$ columns---, we can see that
\begin{align}
\|\epsilon_1 + \epsilon_2\|_2 &\le \frac{2(m-1)\sqrt{d\times k}LL'}{(n-1)(1-\lambda\gamma)} + \frac{2(\lambda\gamma)^{m} \sqrt{d\times k}LL'}{(1-\lambda\gamma)}\label{eq:bt2} 
\end{align}
By concatenating all its columns, the $d \times k$ matrix $G^m_i$ may be seen a single vector $U^m_i$ of size $dk$.
Then, for all $\epsilon>0$,
\begin{align}
\nonumber \P \left(\left\|\frac{1}{n-m}\sum^{n-1}_{i=m}(G^m_i-\E[G^m_i])\right\|_2\geq \epsilon\right) &\leq \P\left(\left\|\frac{1}{n-m}\sum^{n-1}_{i=m}(G^m_i-\E[G^m_i])\right\|_F\geq \epsilon\right) \\
& = \P\left(\left\|\frac{1}{n-m}\sum^{n-1}_{i=m}(U^m_i-\E[U^m_i])\right\|_2 \geq \epsilon\right).\label{eq:Frob_norm}
\end{align}
The variables $U^m_i$ define a stationary $\beta$-mixing process (Lemma~\ref{lemma:beta_relation}). To deal with the $\beta$-mixing assumption, we use the decomposition technique proposed by \citet{Yu:1994} that consists in dividing the stationary sequence $U^m_m,\dots,U^m_{n-1}$ into $2\mu_{n-m}$ blocks of length $a_{n-m}$ (we assume here that $n-m=2a_{n-m}\mu_{n-m}$).
 The blocks are of two kinds: those which contains the even indexes $E=\cup^{\mu_{n-m}}_{l=1}E_l$ and those with odd indexes $H=\cup^{\mu_{n-m}}_{l=1}H_l$. Thus, by grouping the variables into blocks we get
\begin{align}
\P\left(\left\|\frac{1}{n-m}\sum^{n-1}_{i=m}U^m_i-\E[U^m_i]\right\|_2 \ge \epsilon\right) \le & \P \left( \left\|  \sum_{i \in H} U^m_i-\E[U^m_i]\right\|_2 + \left\|  \sum_{i \in E} U^m_i-\E[U^m_i] \right\|_2 \ge (n-m)\frac{\epsilon}{2} \right) \label{eq:7} \\
\nonumber\le & \P \left( \left\|  \sum_{i \in H} U^m_i-\E[U^m_i] \right\|_2 \ge \frac{(n-m)\epsilon}{4} \right) + \\
& ~~~~~~\P \left( \left\|  \sum_{i \in E} U^m_i -\E[U^m_i]\right\|_2 \ge  \frac {(n-m)\epsilon}{4} \right)\label{eq:8} \\
= &2 \P \left( \left\| \sum_{i \in H} U^m_i-\E[U^m_i] \right\|_2 \ge \frac{(n-m)\epsilon}{4} \right)\label{eq:9}
\end{align}
where Equation~\eqref{eq:7} follows from the triangle inequality, Equation~\eqref{eq:8} from the fact that the event $\{X+Y\geq a\}$ implies $\{X\geq \frac{a}{2}\}$ or $\{Y\geq \frac{a}{2}\}$, and Equation~\eqref{eq:9} from the assumption that the process is stationary. Since $H=\cup^{\mu_{n-m}}_{l=1}H_l$ we have
\begin{align}
\P\left(\left\|\frac{1}{n-m}\sum^{n-1}_{i=m}U^m_i-\E[U^m_i]\right\|_2 \ge \epsilon\right) & \leq 2 \P \left( \left\| \sum_{l=1}^{\mu_{n-m}} \sum_{i \in H_l} U^m_i-\E[U^m_i] \right\|_2 \ge \frac{(n-m)\epsilon}{4} \right) \nonumber\\
& = 2 \P \left( \left\| \sum_{l=1}^{\mu_{n-m}} U(H_l)-\E[U(H_l)] \right\|_2 \ge \frac{(n-m)\epsilon}{4} \right)\label{eq:beforedecomp}
\end{align}
where we defined $U(H_l)=\sum_{i \in H_l} U^m_i$. Now consider the sequence of identically distributed independent blocks $(U'(H_l))_{l=1,\dots,\mu_{n-m}}$ such that each block $U'(H_l)$ has the same distribution as $U(H_l)$. We are going to use the following technical result.
\begin{lemma}\cite{Yu:1994}\label{lemma:yu_decomp}
Let $X_1,\dots,X_n$ be a sequence of samples drawn from a stationary
$\beta$-mixing process with coefficients $\{\beta_i\}$.  Let
$X(H)=(X(H_1),\dots,X(H_{\mu_{n-m}}))$ where for all $j$
$X(H_j)=(X_{i})_{i \in H_j}$. Let
$X'(H)=(X'(H_1),\dots,X'(H_{\mu_{n-m}}))$ with $X'(H_j)$
independent and such that for all $j$, $X'(H_j)$ has same distribution
as $X(H_j)$. Let $Q$ and $Q'$ be the distribution of $X(H)$ and
$X'(H)$ respectively.  For any measurable function
$h:\X^{a_n\mu_n}\rightarrow \R$ bounded by $B$, we have
\[|\E_Q[h(X(H)]-\E_{Q'}[h(X'(H)]|\leq B\mu_n\beta_{a_n}.\]
\end{lemma}
By applying Lemma~\ref{lemma:yu_decomp}, Equation~\eqref{eq:beforedecomp} leads to:
\begin{align}
\P\left(\left\|\frac{1}{n-m}\sum^{n-1}_{i=m}U^m_i-\E[U^m_i]\right\|_2 \ge \epsilon\right)\le 2\P \left( \left\| \sum_{l=1}^{\mu_{n-m}}  U'(H_l)-\E[U'(H_l)] \right\|_2 \ge \frac{(n-m)\epsilon}{4} \right) + 2\mu_{n-m} \beta_{a_{n-m}}. \label{eq:10}
\end{align}
The variables $U'(H_l)$ are independent. Furthermore, it can be seen that $(\sum^{\mu_{n-m}}_{l=1}U'(H_l)-\E[U'(H_l)])_{\mu_{n-m}}$ is a $\sigma(U'(H_1),\dots,U'(H_{\mu_{n-m}}))$ martingale:
\begin{align*}
& \E\left[\sum^{\mu_{n-m}}_{l=1}U'(H_l)-\E[U'(H_l)]\left|\vphantom{\sum^{\mu_{n-m}}_{l=1}U'(H_l)-\E[U'(H_l)]}\right.U'(H_1),\dots,U'(H_{\mu_{n-m}-1})\right]\\
= & \sum^{\mu_{n-m}-1}_{l=1}U'(H_l)-\E[U'(H_l)]+\E[U'_{H_{\mu_{n-m}}}-\E[U'_{H_{\mu_{n-m}}}]]\\
= & \sum^{\mu_{n-m}-1}_{l=1}U'(H_l)-\E[U'(H_l)].
\end{align*}
We can now use the following concentration result for martingales.
\begin{lemma}[\cite{hayes}]
\label{lemma:hayes}
Let $X=(X_0,\dots,X_n)$ be a discrete time martingale taking values in an Euclidean space such that $X_0=0$ and for all $i$, $\|X_i-X_{i-1}\|_2 \le B_2$ almost surely. Then for all  $\epsilon$, 
$$
P \left\{ \|X_n\|_2 \ge \epsilon \right\} < 2e^2 e^{-\frac{\epsilon^2}{2n(B_2)^2}}. 
$$
\end{lemma}
Indeed, taking $X_{\mu_{n-m}}=\sum_{l=}^{\mu_{n-m}} U'(H_l)-\E[U'(H_l)]$, and observing that $\|X_i-X_{i-1}\|=\|U'(H_l)-\E[U'(H_l)]\|_2\leq a_{n-m} C$ with $C=\frac{2 \sqrt{dk} LL'}{1-\lambda\gamma}$, the lemma leads to
\begin{align*}
\P \left( \left\|  \sum^{\mu_{n-m}}_{l=1}U'(H_l)-\E[U'(H_l)]  \right\|_2 \ge \frac{(n-m)\epsilon}{4} \right)& \le 2e^2 e^{-\frac{(n-m)^2\epsilon^2}{32 \mu_{n-m} (a_{n-m}C)^2}} \\
& = 2e^2 e^{-\frac{(n-m)\epsilon^2}{16 a_{n-m} C^2}}.
\end{align*}
where the second line is obtained by using the fact that $2a_{n-m}\mu_{n-m}=n-m$.
With Equations~\eqref{eq:beforedecomp} and \eqref{eq:10}, we finally obtain
\begin{align*}
\P\left(\left\|\frac{1}{n-m}\sum^{n-1}_{i=m}U^m_i-\E[U^m_i]\right\|_2 \ge \epsilon\right)& \le 4 e^2 e^{-\frac{(n-m)\epsilon^2}{16 a_{n-m} C^2}} + 2(n-m) \beta^U_{a_{n-m}}.
\end{align*}

The vector $U^m_i$ is a function of $Z_i=(X_{i-m+1},\dots,X_{i+1})$, and Lemma~\ref{lemma:beta_relation} tells us that for all $j>m$,
\begin{align*}
\beta_j^U \le \beta^Z_j \leq \beta^X_{j-m}\le \overline\beta e^{-b (j-m)^\kappa}. 
\end{align*} 
So the equation above may be re-written as
\begin{align}
\P\left(\left\|\frac{1}{n-m}\sum^{n-1}_{i=m}U^m_i-\E[U^m_i]\right\|_2 \ge \epsilon\right)  \le   4 e^2 e^{-\frac{(n-m)\epsilon^2}{16 a_{n-m} C^2}} + 2(n-m)\overline\beta e^{-b (a_{n-m}-m)^\kappa}=\delta'.\label{eq:mixingvector}
\end{align}
We now follow a reasoning similar to that of \cite{lazaric:2012} in order to get the same exponent in both of the above exponentials. Taking
$a_{n-m}-m=\left\lceil \frac{C_2(n-m)\epsilon^2}{b}\right\rceil^{\frac{1}{\kappa+1}}$ with $C_2=(16C^2\zeta)^{-1}$, and $\zeta=\frac{a_{n-m}}{a_{n-m}-m}$, we have
\begin{align}
\delta'\leq (4e^2+(n-m)\overline \beta)\exp\left(-\min\left\{\left(\frac{b}{(n-m)\epsilon^2C_2}\right),1\right\}^{\frac{1}{k+1}}\frac{1}{2}(n-m)C_2\epsilon^2\right).\label{eq:parameter_delta'}
\end{align}
Define
\[\Lambda(n,\delta)=\log\left(\frac{2}{\delta}\right)+\log(\max\{4e^2,n\overline\beta\}),\] and 
\[\epsilon(\delta)=\sqrt{2\frac{\Lambda(n-m,\delta)}{C_2(n-m)}\max\left\{\frac{\Lambda(n-m,\delta)}{b},1\right\}^{\frac{1}{\kappa}}}.\] 
It can be shown that
\begin{equation}
\label{lemma:fonction_Lambda}
\exp\left(-\min\left\{\left(\frac{b}{(n-m)(\epsilon(\delta))^2C_2}\right),1\right\}^{\frac{1}{k+1}}\frac{1}{2}(n-m)C_2(\epsilon(\delta))^2\right)\leq\exp\left(-\Lambda(n-m,\delta) \right).
\end{equation}
Indeed\footnote{This inequality exists in \cite{lazaric:2012}, and is developped here for completeness.}, there are two cases:
\begin{enumerate}
\item Suppose that  $\min\left\{\left(\frac{b}{(n-m)(\epsilon(\delta))^2C_2}\right),1\right\}=1$. Then
\begin{align*}
& \exp\left(-\min\left\{\left(\frac{b}{(n-m)(\epsilon(\delta))^2C_2}\right),1\right\}^{\frac{1}{k+1}}\frac{1}{2}(n-m)C_2(\epsilon(\delta))^2\right) \\
= & \exp\left(-\Lambda(n-m,\delta)\max\left\{\frac{\Lambda(n-m,\delta)}{b},1\right\}^{\frac{1}{k}}\right)\\
\leq & \exp\left(-\Lambda(n-m,\delta)\right).
\end{align*}
\item Suppose now that $\min\left\{\left(\frac{b}{(n-m)(\epsilon(\delta))^2C_2}\right),1\right\}=\left(\frac{b}{(n-m)(\epsilon(\delta))^2C_2}\right)$. Then
\begin{align*}
\exp\left(-\frac{1}{2}b^{\frac{1}{k+1}}((n-m)C_2(\epsilon(\delta))^2)^\frac{k}{k+1}\right)&=\exp\left(-\frac{1}{2}b^{\frac{1}{k+1}}(\Lambda(n-m,\delta)^\frac{k}{k+1}\max\left\{\frac{\Lambda(n-m,\delta)}{b},1\right\}^{\frac{1}{k+1}}\right)\\
&=\exp\left(-\frac{1}{2}\Lambda(n-m,\delta)^\frac{k}{k+1}\max\left\{\Lambda(n-m,\delta),b\right\}^{\frac{1}{k+1}}\right)\\
&\leq \exp\left( -\Lambda(n-m,\delta)\right). 
\end{align*}
\end{enumerate}
By combining Equations~\eqref{eq:parameter_delta'} and~\eqref{lemma:fonction_Lambda}, we get
\[\delta'\leq (4e^2+(n-m)\overline \beta)\exp\left(-\Lambda(n-m,\delta) \right).\] 
If we replace $\Lambda(n-m,\delta)$ with its expression, we obtain
\[\exp\left(-\Lambda(n-m,\delta) \right)=\frac{\delta}{2} \max\{4e^2,(n-m)\overline \beta\}^{-1}.\]
Since $4e^2\max\{4e^2,(n-m)\overline \beta\}^{-1}\leq 1$ and  $(n-m)\overline{\beta}\max\{4e^2,(n-m)\overline \beta\}^{-1}\leq 1$, we consequently have
\[\delta'\leq 2\frac{\delta}{2}\leq \delta.\]
Now, note that since $a_{n-m}-m\geq 1$, we have
\begin{align*}
\zeta&= \frac{a_{n-m}}{a_{n-m}-m}=\frac{a_{n-m}-m+m}{a_{n-m}-m}\leq 1+m.
\end{align*}
Let $J(n,\delta)=32\Lambda(n,\delta)\max\left\{\frac{\Lambda(n,\delta)}{b},1\right\}^{\frac{1}{\kappa}}$. Then Equation~\eqref{eq:mixingvector} is reduced to 
\begin{align}
\P\left(\left\|\frac{1}{n-m}\sum^{n-1}_{i=m}{(U^m_i-\E[U^m_i])}\right\|_2\ge\frac{C}{\sqrt{n-m}}\left(\zeta J(n-m,\delta)\right)^{\frac{1}{2}}\right)\leq \delta.\label{eq:ineg_U}
\end{align}
Since $J(n,\delta)$ is an increasing function on $n$, and $\frac{n-1}{\sqrt{n-1}(n-m)}=\frac{1}{\sqrt{n-m}}\sqrt{\frac{n-1}{n-m}} \ge \frac{1}{\sqrt{n-m}}$, we have
\begin{align*}
&\P\left(\left\|\frac{1}{n-1}\sum^{n-1}_{i=m}(G^m_i-\E[G^m_i])\right\|_2\geq \frac{C}{\sqrt{n-1}}\left(\zeta J(n-1,\delta)\right)^{\frac{1}{2}}\right)\\
\leq~ & \P\left(\left\|\frac{1}{n-m}\sum^{n-1}_{i=m}(G^m_i-\E[G^m_i])\right\|_2\geq \frac{C}{\sqrt{n-1}}\frac{n-1}{n-m}\left((m+1) J(n-1,\delta)\right)^{\frac{1}{2}}\right)\\ 
\leq~ &\P\left(\left\|\frac{1}{n-m}\sum^{n-1}_{i=m}(G^m_i-\E[G^m_i])\right\|_2\geq \frac{C}{\sqrt{n-m}}\left((m+1) J(n-m,\delta)\right)^{\frac{1}{2}}\right).
\end{align*}
By using Equations~\eqref{eq:Frob_norm} and~\eqref{eq:ineg_U}, we deduce that
\begin{align}
\P\left(\left\|\frac{1}{n-1}\sum^{n-1}_{i=m}(G^m_i-\E[G^m_i])\right\|_2\geq \frac{C}{\sqrt{n-1}}\left((m+1)J(n-1,\delta)\right)^{\frac{1}{2}}\right)\leq\delta.\label{eq:bt3}
\end{align}
By combining Equations~\eqref{eq:bt1}, \eqref{eq:bt2},\eqref{eq:bt3}, plugging the value of $C=\frac{2\sqrt{dk}LL'}{1-\lambda\gamma}$, and  taking  $m=\left\lceil\frac{\log{(n-1)}}{\log \frac{1}{\lambda\gamma}}\right\rceil$, we get the announced result.

\bibliographystyle{natbib2}
\bibliography{biblio.bib}

\end{document}